\documentclass{article}
\usepackage{ijcai23}
\usepackage{times}
\usepackage{soul}
\usepackage{url}
\usepackage[hidelinks]{hyperref}
\usepackage[utf8]{inputenc}
\usepackage[small]{caption}
\usepackage{graphicx}
\usepackage{amsmath}
\usepackage{amsthm}
\usepackage{booktabs}
\usepackage{algorithm}
\usepackage[switch]{lineno}

\usepackage{algpseudocode}
\usepackage{multirow}
\usepackage{amsmath} 
\usepackage{xcolor}
\usepackage{amssymb}
\usepackage{amsmath}
\usepackage{subfigure}
\usepackage{xspace}
\usepackage{multibib}
\usepackage{tabularx}
\usepackage{makecell}
\usepackage{threeparttable}
\usepackage{IEEEtrantools}
\usepackage{enumitem}
\usepackage{newfloat}
\usepackage{listings}
\usepackage{colortbl}
\usepackage{bm}
\newcommand{\agr}{\textnormal{AGR}}
\newcommand{\ourmethod}{FMPA\xspace}
\newcommand{\ct}[1]{\texttt{#1}}

\newtheorem{example}{Example}
\newtheorem{theorem}{Theorem}
\newtheorem{definition}{Definition}

\newtheorem{lemma}{Lemma}
\title{Denial-of-Service or Fine-Grained Control: Towards Flexible Model Poisoning Attacks on Federated Learning}

\author{
Hangtao Zhang$^{1}$\and
Zeming Yao$^2$\and
Leo Yu Zhang$^3$\and
Shengshan Hu$^{1,4,5,6,7}$\and
Chao Chen$^8$\and \\
Alan Liew$^3$\And
Zhetao Li$^9$
\affiliations
$^1$School of Cyber Science and Engineering, Huazhong University of Science and Technology\\
$^2$Swinburne University of Technology\\
$^3$Griffith University\\
$^4$National Engineering Research Center for Big Data Technology and System\\
$^5$Services Computing Technology and System Lab\\
$^6$Hubei Key Laboratory of Distributed System Security\\
$^7$Hubei Engineering Research Center on Big Data Security\\
$^8$RMIT University\\
$^9$Xiangtan University
\emails
\{zhanghangtao7, nick.yao.zm\}@gamil.com,
leo.zhang@griffith.edu.au,
hushengshan@hust.edu.cn,
chao.chen@rmit.edu.au,
a.liew@griffith.edu.au,
liztchina@hotmail.com
}

\begin{document}

\maketitle

\begin{abstract}  %no more than 200 words and exactly 200 words now
    \textit{Federated learning} (FL) is vulnerable to poisoning attacks, where adversaries corrupt the global aggregation results and cause \textit{denial-of-service} (DoS). 
    Unlike recent model poisoning attacks that optimize the amplitude of malicious perturbations along certain prescribed directions to cause DoS, we propose a \underline{\textbf{F}}lexible \underline{\textbf{M}}odel \underline{\textbf{P}}oisoning \underline{\textbf{A}}ttack (\ct{\ourmethod}) that can achieve versatile attack goals.
    We consider a practical threat scenario where no extra knowledge about the FL system (e.g., aggregation rules or updates on benign devices) is available to adversaries. \ct{\ourmethod} exploits the global historical information to construct an estimator that predicts the next round of the global model as a benign reference. It then fine-tunes the reference model to obtain the desired poisoned model with low accuracy and small perturbations. 
    Besides the goal of causing DoS, \ct{\ourmethod} can be naturally extended to launch a fine-grained controllable attack, making it possible to precisely reduce the global accuracy. 
    Armed with precise control, malicious FL service providers can gain advantages over their competitors without getting noticed, hence opening a new attack surface in FL other than DoS. Even for the purpose of DoS, experiments show that \ct{\ourmethod} significantly decreases the global accuracy, outperforming six state-of-the-art attacks.
\end{abstract}

\begin{figure*}[htp]
    \centering
\includegraphics[width=0.6\textwidth]{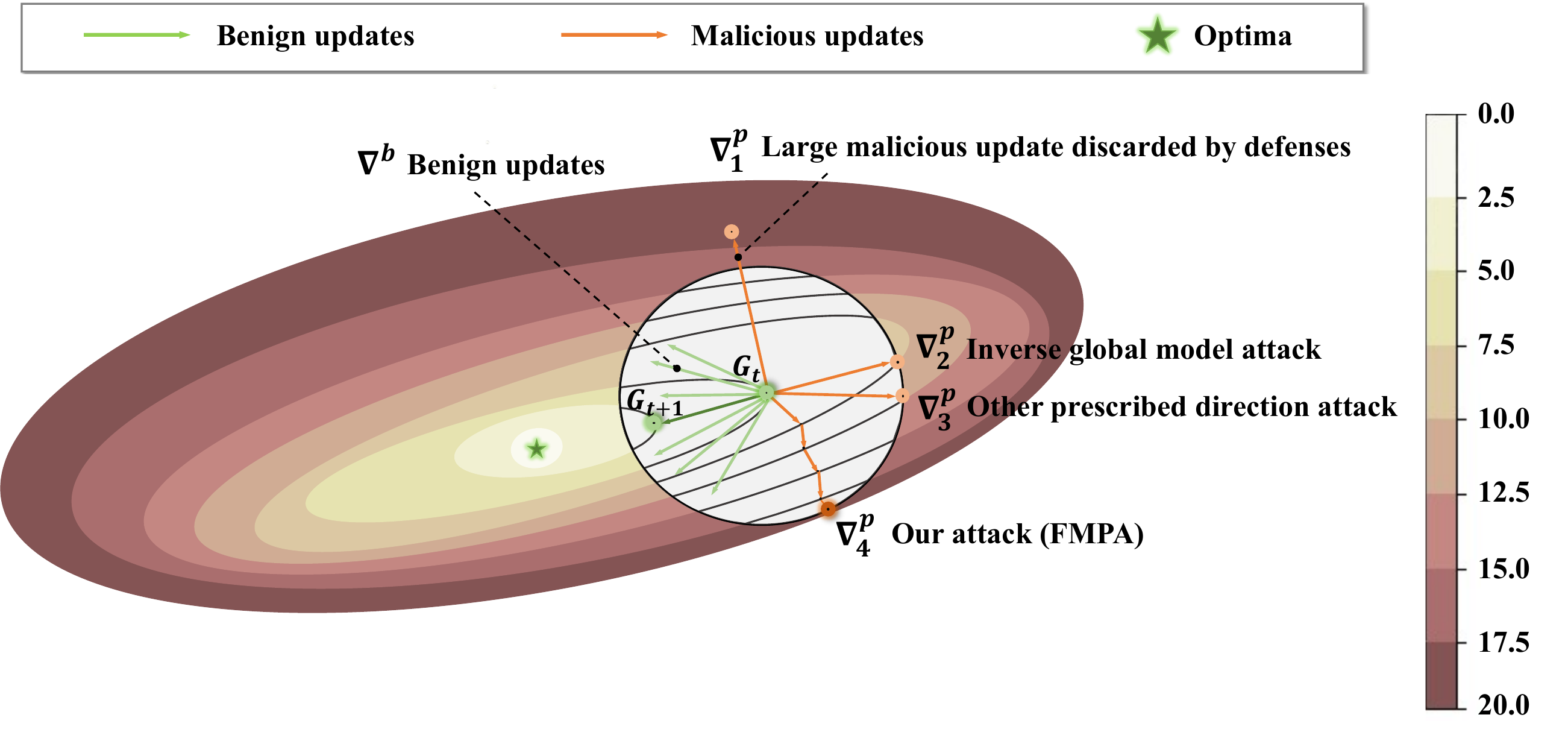}
    \caption{\textbf{A schematic of our attack:} \bm{$G_t$} and \bm{$G_{t+1}$} represent the global model for for rounds $t$ and $t+1$, respectively. The gray circle indicates the detection scope of a specific defense, thus red arrows (malicious local updates) outside the scope will be discarded (e.g., \bm{$\nabla_1^p$}).
    The red arrows (e.g., \bm{$\nabla_2^p$} and \bm{$\nabla_3^p$}) represent current wisdom used for designing malicious updates, where our \ct{\ourmethod} (\bm{$\nabla_4^p$}) can achieve the maximum loss (best attack effect) with the same amount of perturbation.}
    \label{Fig:SchematicDigram} 
\end{figure*}

\section{Introduction}

\textit{Federated learning} (FL)~\cite{GoogleAIBlog} has recently emerged as a new distributed learning paradigm. It enables multiple clients (e.g., mobile devices) to jointly learn a global model without revealing their privacy-sensitive data. Specifically, in FL, numerous clients can train local models on their private datasets and iteratively upload models/updates to the central server (e.g., Apple, Google). The server will aggregate their updates to obtain global model updates using an \textit{aggregation algorithm} (AGR). 
% FL has been widely used in biomedicine~\cite{SantiagoSilva2020FedBioMedAG}, finance~\cite{GuodongLong2021FederatedLF} and autonomous driving ~\cite{YijingLi2021PrivacyPreservedFL}, etc.

However, Due to its decentralized nature, FL is susceptible to poisoning attacks. By poisoning local data, malicious clients (i.e., clients compromised by an adversary) can drive the global model to learn the wrong knowledge or embed backdoors (known as \textit{data poisoning attack})~\cite{ValeTolpegin2020DataPA,HongyiWang2020AttackOT}. Or they can submit malicious models to the central server to disrupt the training process (a.k.a. \textit{model poisoning attack})~\cite{XingchenZhou2021DeepMP,pmlr-v115-xie20a}, such that the learnt model suffers from high testing error indiscriminately and eventually causes a \textit{denial-of-service} (DoS).

These attacks can be either untargeted, with the goal of reducing the overall quality of the learnt model to cause DoS, or targeted, controlling the model to misclassify samples into an adversary’s desired class (known as \textit{backdoor attack}). \textbf{Our work focuses on the untargeted model poisoning attack (MPA), but the outcome is not merely causing DoS. It opens up a new attack surface (dubbed fine-grained control attacks) in FL that is more stealthy than DoS.}

To mitigate MPAs in FL, a growing number of defense options have been proposed ~\cite{XiaoyuCao2020FLTrustBF,SauravPrakash2020MitigatingBA}, where the server employs a robust AGR. Their main approach is to compare the clients' local updates and exclude abnormal ones before aggregation. Recent studies ~\cite{MinghongFang2019LocalMP,ViratShejwalkar2022ManipulatingTB} show further adaptive or optimized MPAs can still bypass robust AGRs, especially when AGRs are known to the adversary. 
These novel MPAs have a huge impact on FL compared to traditional \textit{data poisoning attacks}.

However, most existing advanced MPAs suffer from (part of) the following limitations:
\begin{itemize}[leftmargin=*]
    \item Assumptions about the adversary are often strong. For instance, the adversary knows the AGRs adopted by the central server~\cite{MinghongFang2019LocalMP}, or can assess all benign clients' local model updates (or their compele statistics)~\cite{GiladBaruch2019ALI}. 
    Practically, without these assumptions, the attack would be greatly weakened.
    \item Advanced MPAs~\cite{ViratShejwalkar2022ManipulatingTB} commonly craft poisoned updates via maximizing the distance between poisoned and benign updates towards a prescribed direction (e.g., inverse of the global model). 
    However, the prescribed direction itself does not guarantee good attack performance, which in turn makes the advanced MPAs built upon it less optimal.
    \item The adversary cannot precisely control the impact of the submitted poisoned model updates. To cause DoS, the adversary needs to craft poisoned updates farthest from the global model, which makes the attack easier to detect.
\end{itemize}

To address aforementioned challenges, we introduce a novel (untargeted) \underline{\textbf{F}}lexible \underline{\textbf{M}}odel \underline{\textbf{P}}oisoning \underline{\textbf{A}}ttack (\ct{\ourmethod}). \ct{\ourmethod} works efficiently without additional information about the FL system, except for the global model that all clients receive during normal FL training. The high-level approach of \ct{\ourmethod} is to find the current benign reference model by exploiting the key information from the historical global models.
In contrast to optimizing for the best perturbation magnitude along a prescribed malicious direction (e.g., inverse of the global model, inverse standard deviation) suggested in existing works ~\cite{MinghongFang2019LocalMP,ViratShejwalkar2022ManipulatingTB}, \ct{\ourmethod} fine-tunes the reference model to directly obtain a malicious model with low accuracy and small perturbations. Fig.~\ref{Fig:SchematicDigram} intuitively illustrates the difference between \ct{\ct{\ourmethod}} and existing attacks through visualizing the loss landscape.

Our contributions are highlighted as follows:
\begin{itemize}[leftmargin=*]

    \item We introduce a new MPA for FL risk assessment. Local poisoning models are designed by attacking vulnerable parts of neural networks and minimizing the degree of modification to evade detection of Byzantine-robust AGRs. Extensive experiments reveal that, \ct{\ourmethod} achieves better attack impact, albeit with weaker and more realistic assumptions on the adversary.
    
    \item Historical data in FL has been recently employed for detecting anomalous updates~\cite{Zhang2022FLDetectorDF}. To the best of our knowledge, our \ct{\ourmethod} is the first MPA to focus on both the vertical and horizontal information of FL systems and exploit them to assist poisoning attacks.
    
    \item We introduce a new concept of model poisoning to the community, which is not to merely target for degradation of the global model's performance indiscriminately (a.k.a. DoS), but to conceal malicious clients under a precise control of the extent of the attack. 
    This poses an even greater threat to commercialized FL applications since an adversary can inject this subtle attack into peer FL systems to gain an advantage in competing FL services without getting noticed.
\end{itemize}

\section{Background and Related Work}

\subsection{Federated Learning}

Assume there is a general FL system, comprising of $N$ clients and a central server. Each client $i\in[N]$ has a local training data set $\mathcal{D}_{i}$. We define the vector $g \in \mathbb{R}^d$ as the parameters of the co-learned model, while $j(g;c)$ denotes the loss function (e.g., \textit{cross-entropy}) of $g$ on input-output pairs $c$. Formally, the optimization objective of client $i$ is $J_i(g)=\mathbb{E}_{c \sim \mathcal{D}_i}[j(g ; c)]$. FL aims to find a model $g^{\star}$ that minimizes the sum of losses among clients:

\begin{IEEEeqnarray}{rCL}
    g^{\star} = \min _{g \in \mathbb{R}^d}\left[J(g):=\sum_{i=1}^N \beta_i^t J_i(g)\right],
\end{IEEEeqnarray}
where client $i$ is weighted by $\beta_i^t>0$ in the $t$-th iteration.

Specifically, FL executes the following three steps: 
\begin{itemize}[leftmargin=*]

      \item \textbf{Step I:} In the $t$-th iteration, the global model $g^{t}$ is synchronized by the central server to all clients or a subset of them. 
      
      \item \textbf{Step II:} After receiving $g^{t}$, the $i$-th client trains a new local model $w_{i}^{t}$ over $\mathcal{D}_{i}$ by solving the optimization problem 
      $\arg\min_{w_{i}^{t}}J_{i}(w_{i}^{t})$ and then send its local model update $ \nabla_{i}^{t}=w_{i}^{t}-g^{t}$ to the server.
      
      \item \textbf{Step III:} The central server aggregates all the local model updates based on an aggregation algorithm $F_{\agr}(\cdot)$ (e.g., FedAvg~\cite{HBrendanMcMahan2016CommunicationEfficientLO}) to obtain global update for this round as $\nabla^{t}=F_{\agr}(\nabla_{i\in{[N]}}^{t})$.
      
      Finally, the server updates the global model, i.e., $g^{t+1}=g^{t}-\eta \cdot \nabla^{t}$, where $\eta$ is the global learning rate.
\end{itemize}

\subsection{Poisoning Attacks on FL}
\label{Sec:AttackIntroduce}
Poisoning attacks are divided into two categories, namely \textit{model poisoning} and \textit{data poisoning} attack. This paper focuses on the more dangerous MPAs.

\subsubsection{Model Poisoning Attack (MPA)} MPA is more dangerous as the adversary can directly manipulate local model updates. Previous MPAs include sign flipping attacks~\cite{LipingLi2019RSABS}, Gaussian attacks~\cite{CongXie2018GeneralizedBS}, etc.
But these attacks cause a significant difference between poisoned and benign updates, which makes it easier to detect. 
As discussed below, state-of-the-art (SOTA) MPAs are either adaptive-based or optimization-based.

\ct{LIE}~\cite{GiladBaruch2019ALI} computes the mean $\mu$ and standard deviation $\sigma$ of all updates, and finds the 
fixed coefficient $z$ based on the total number of clients. The final generated malicious update is a weighted sum: $\widetilde{\nabla} = \mu + z\sigma$. But \ct{LIE} unrealistically assumes complete knowledge of all clients' updates.
Under the assumption that the exact defenses are known to adversaries,~\cite{MinghongFang2019LocalMP} tailored attacks for different defenses by formulating it as an optimization problem of finding the optimized poisoned updates against known defenses. 
Derived from this,~\cite{ViratShejwalkar2022ManipulatingTB} proposes \ct{Min-Max} and \ct{Min-Sum}, which minimize the maximum and the sum of distances between poisoned and all the benign updates, respectively. As reported in~\cite{ViratShejwalkar2022ManipulatingTB}, choosing an appropriate perturbation direction for the malicious update $\widetilde{\nabla}$ is the key to an effective MPA.
Different heuristic directions, e.g., inverse standard deviation or inverse sign, are used by them, while our work directly investigates the optimal $\widetilde{\nabla}$.
~\cite{pmlr-v115-xie20a} devised a MPA known as inner production manipulation (\ct{IPM}). \ct{MPAF}~\cite{XiaoyuCao2023MPAFMP} launches an attack based on fake clients injected into the FL system. But both \ct{IPM} and \ct{MPAF} are limited in breaking a few simple defenses.

%We will use the following Byzantine-robust FL AGRs to evaluate MPAs.
\subsection{Existing Byzantine-Robust AGRs}

\label{Sec:Defense}
We divide the current defenses into three categories, namely \textit{statistics-based}, \textit{distance-based} and \textit{performance-based}, depending on the principles that the server relies to detect or evade suspicious models. \textit{Distance-based} defenses focus on comparing the distance between local updates to find outliers. \textit{Statistics-based} defenses utilize statistical features to remove statistical outliers.  \textit{Performance-based} defenses depend on a validation dataset to evaluate the quality of uploaded models. 
But some recent \textit{performance-based} defenses are screened by evaluating the performance of local models with the need for the central server to use a clean dataset~\cite{WeiWan2022ShieldingFL,XiaoyuCao2020FLTrustBF}. This violates the privacy principle of FL, and similar to other Byzantine-Robust AGRs like~\cite{pmlr-v139-karimireddy21a,ViratShejwalkar2022ManipulatingTB}, we do not evaluate these defenses in detail.

\section{Flexible Model Poisoning Attack (FMPA)}

\label{Sec:Method}
In this section, before presenting our \ct{FMPA}, the threat model is discussed.

\paragraph{Adversary’s Capabilities.} Suppose an adversary controls $m$ of total $N$ clients 
(i.e., malicious clients). We assume that the malicious clients in all adversarial settings are below 50\% (i.e., $(m/N)<0.5$), otherwise the adversary will easily manipulate FL. The adversary can obtain the global model broadcasted each round and manipulate the local updates on malicious clients.

\paragraph{Adversary’s Knowledge.} 
We assume a weaker adversary who neither knows the exact AGR used by the server nor any information from other $(N-m)$ benign clients (e.g., models/gradients or local data). The adversary only has access to the data of the malicious clients as in normal FL training.

\paragraph{Adversary’s Goals.} Similar to existing MPAs, some attackers want to indiscriminately downgrade the accuracy of the global model at its testing time.
Different from existing studies, some attackers may want to precisely control the effect of poisoning while maintaining concealment so as to gain advantages in establishing a competing FL service. 

Next, we will introduce the method of designing concealed and effective malicious models in FL. We start with an analysis to show how to build the general model poisoning framework for deep classification models and then adapt it to flexibly poison FL systems to meet the above goals.

\subsection{Model Poisoning Problem Formulation}

Consider a classification task with $L$ classes and a DNN model $\Theta$. For a clean input-output pair $\left({x}_{i}, y_{i}\right)$, we denote the output probability for the $l$-th category in a neuron network as $Z_{l}\left({x_i},  \Theta \right)$, where $x_i \in \mathcal{X}$ and its ground-truth label $y_i \in \mathcal{Y}$. 
Predictions for a test sample $x$ are usually measured using the \textit{softmax function}:
\begin{equation}
P_l\left(\Theta\right)=Z_{l}\left({x}, 
\Theta\right)=\operatorname{softmax}(v(x)),
\end{equation}
where $P = [P_l]_{l=1}^L$ represents the confidence score vector of each category. 
Here, $P_{l} \geq 0$, $\sum_{l\in [L]} P_{l}=1$ and $v(x)$ is the logit vector of the penultimate layer of the neural network. Finally, the predicted label of the test sample $(x,y)$ is $\hat{y}=\arg \max _{l \in[1, L]} P_{l}$. Hereinafter, we follow the MPA literature to focus on image classification tasks for easy presentation.

The ultimate goal of any untargeted MPA is to add adversarial perturbations to network parameters to make the inference results deviate from the true labels of the images. This is similar to the case of adversarial attack~\cite{ChristianSzegedy2013IntriguingPO,ZekunZhang2020LearningOT}, with the difference that perturbations are added to the input images. Inspired by this observation, we aim to design an appropriate adversarial objective function to account for creating malicious models by gradient-based optimization. 
Generally, any untargeted attack can be defined as the following constrained optimization problem:
\begin{equation}
\begin{aligned}
&\min &&\mathcal{H}\left(\Theta, \Theta+\delta_p\right),  \\
&\text{ s.t. }  && y \neq \hat{y} =  \arg \max_{l} Z_{l}(x, \Theta+\delta_p),
\end{aligned}
\end{equation}
where $\mathcal{H}$ is a measure of the distance (e.g., $\ell_{0}, \ell_{2}$, or $\ell_{\infty}$ distance) between the benign and malicious models, $\delta_p$ is a perturbation crafted by the adversary to poison the model $\Theta$.

To address the difficulty of directly solving the above optimization problem (\textit{non-convex constraints} and \textit{non-linear}), we adopt the heuristic from \cite{carlini2017towards} to express it in a form that is suitable for gradient-based optimization. Say $\Theta^{\prime} = \Theta+\delta_p$, we define a new loss function $\mathcal{L}\left(\Theta^{\prime} \right)=\mathcal{L}(\Theta+\delta_p)$ such that the constraint $ y\neq \hat{y} = \arg \max _{l} Z_{l}(x, \Theta^{\prime})$ is satisfied if and only if $\mathcal{L}\left(\Theta^{\prime}\right)\leqslant0$. We use a variant \textit{hinge loss} to characterize this requirement, i.e., 

\begin{IEEEeqnarray}{rCL}
\mathcal{L}\left(\Theta^{\prime}\right) := \max \left(0, \ P_{y}(\Theta^{\prime})-\max _{j \neq  y} P_{j}(\Theta^{\prime}) \right),
\end{IEEEeqnarray}
where penalties are incurred when the entry of the ground-truth label is the largest among all labels.

Use the typical $\ell_{2}$-norm to measure the similarity between benign and poisoned models, the untargeted MPA problem can be represented by the following equation:
\begin{equation}
\min \lambda \cdot\|\delta_p\|_{2}+\mathcal{L}(\Theta+\delta_p),
\label{Eq:My_loss}
\end{equation}
where $\lambda$ is a scalar coefficient to control the relative importance of perturbation. The perturbation $\delta_p$ (and the malicious model $\Theta^{\prime}$) can be now solved by gradient-based optimization through backpropagation.

The above formulation of the general untargeted MPA reveals the key fact that, regardless of centralized or FL, the desired model perturbation $\delta_p$ (and thus a poisoned model) can be directly optimized. This is how the proposed \ct{FMPA} differentiates from the existing works \cite{ViratShejwalkar2022ManipulatingTB,GiladBaruch2019ALI,MinghongFang2019LocalMP}, where the standard deviation noise or the inverse of the global model (refer to Fig.~\ref{Fig:SchematicDigram} for visualization) is used while optimization is enforced to the magnitude of the prescribed direction perturbation to bypass defenses.

\subsection{Optimization-Based Malicious Model Update}
Instead of poisoning a fixed model $\Theta$ as shown above, poisoning of FL is a sequential process along its entire training. To apply the new formulation to FL, the adversary first needs a \textbf{reference model} as a starting point to craft a poisoned update. The basic principle to be followed is that an ideal reference model should approximate the aggregation result $F_{\agr}(w_{\{i\in[N]\}})$ in the absence of attacks. Next, we illustrate how an attacker can obtain an appropriate reference model.

\paragraph{Different Reference Models.} We discuss three reference models, namely \textit{historical reference model} (HRM), \textit{alternative reference model} (ARM) and \textit{predictive reference model} (PRM). 
HRM is a naive approach that takes the historical global model $g^t$ of round $t$ as the reference.
ARM is usually adopted in existing MPAs~\cite{MinghongFang2019LocalMP,ViratShejwalkar2022ManipulatingTB}, and it actually simulates the average aggregation of all $m$ local updated models $w_{\{i\in[m]\}}$ available to the attacker, i.e., $\frac{1}{m}\sum_{i \in[m]} w_{i}$. It is clear that the quality of ARM depends on the number $m$ of compromised clients and their local datasets.
PRM provides a better way of computing the reference model and does not depend on any other conditions. It predicts ${g}^{t+1}$ from global historical information as discussed below.

We use exponential smoothing~\cite{RobertGBrown1961TheFT}, a well-known lightweight forecasting algorithm for time series, to build a PRM. Assume that $g^{t}$ is global model in the $t$-th round, $s_{t}^{(1)}$ and $s_{t}^{(2)}$ are the first and second order exponential smoothing values for round $t$, respectively.
According to \textit{Taylor series} and \textit{Exponential smoothing}, we calculates the estimator $\widehat{g}^{t+1}$ as the reference model:
\begin{IEEEeqnarray}{rCL}
\widehat{g}^{t+1}=\mathbb{P}(s_{t}^{(1)},s_{t}^{(2)})=\frac{2-\alpha}{1-\alpha} s_{t}^{(1)}-\frac{1}{1-\alpha} s_{t}^{(2)},
\label{Eq:Predict}
\end{IEEEeqnarray}
where 
$\begin{cases} 
s_{t}^{(1)}=\alpha g^{t}+(1-\alpha)s_{t-1}^{(1)}\\
s_{t}^{(2)}=\alpha s_{t}^{(1)}+(1-\alpha)s_{t-1}^{(2)}, \end{cases}$ 
and $\alpha\in(0, 1)$ adjusts the portion of the latest global model. We take the average of the ﬁrst several true values to initialize $s_{0}^{(1)}$ and $s_{0}^{(2)}$, and recursively computes $s_{t}^{(1)}$ and $s_{t}^{(2)}$. Then we can easily update  $\widehat{g}^{t+1}$ using only $s_{t}^{(1)}$ and $s_{t}^{(2)}$.

It is worth noting that unlike the literature, prediction of  ${g}^{t+1}$ is used for detecting malicious clients \cite{Zhang2022FLDetectorDF}, we are the first to use it to aid in designing MPA. As will be validated in the experiments, PRM achieved the most stable poison effect among the three reference models.

\subsubsection{Indiscriminate Untargeted Attack (\ct{I-FMPA})}

In this scenario, for each training round of FL, malicious clients upload crafted poisoned updates to affect the
to-be-aggregated central model. Note that we keep all poisoned updates the same to maximize the attack effect.
Taking $\widehat{g}^{t+1}$ as the reference model, malicious model can be crafted by solving Eq.~(\ref{Eq:My_loss}) using \textit{stochastic gradient descent}. Algorithm~\ref{alg:algorithm1} gives a complete description of \ct{I-FMPA}.

To cause DoS, the best attack performance that untargeted MPAs can achieve is random guessing. To maximize attack impact, for a classification model with $L$ categories, we set the accuracy threshold $\tau$ in Algorithm~\ref{alg:algorithm1} to be $(\frac{1}{L})$.
First, the adversary divides the available data $\mathbb{D}$ on compromised clients into a training set $\mathbb{D}^{\textnormal{train}}$ and a validation set $\mathbb{D}^{\textnormal{val}}$. 
In {lines 5-6}, the adversary initializes the malicious model by the reference model $\widehat{g}^{t+1}$. In {lines 8-9}, the adversary finds a \textbf{certified radius} $\mathcal{R}$, which is subsequently used to tailor malicious updates to a suitable size for circumventing defenses.
In {lines 13-14}, the adversary trains a fresh malicious model $\Theta^{\prime}$ with the optimization objective of Eq.~(\ref{Eq:My_loss}). 
The desired accuracy threshold $\tau$ is ensured through validating on $\mathbb{D}^{\textnormal{val}}$ in line 15. In line 21, the adversary conceals itself by projecting the update on a ball of radius $\delta$ ($\|\delta\|_2<\mathcal{R}$) around $\overline{\nabla}^{t}$.

\renewcommand{\algorithmicrequire}{\textbf{Input:}} 
\renewcommand{\algorithmicensure}{\textbf{Output:}} 

\begin{algorithm}[tb]
    \caption{Generating Malicious Model (Update)}
    \label{alg:algorithm1}
    \textbf{Input}: $m$ malicious clients, global model $g^t\in \mathbb{R}^d$ broadcasted at round $t$, accuracy threshold $\tau$, malicious clients data $\mathbb{D}=\left\{\mathcal{D}_1, \ldots, \mathcal{D}_{m}\right\}$ available to the adversary, the first and second order exponential smoothing {$s_{t}^{(1)}$, $s_{t}^{(2)}$}.\\  
    \textbf{Function}: $\mathcal{A}(\cdot)$: optimizer for parameter updates  \\
    $\mathbb{P}(\cdot,\cdot)$: Predict the next-round model ($g^{t}\rightarrow \widehat{g}^{t+1}$) as Eq.~(\ref{Eq:Predict})\\
    \textbf{Output}: malicious model $\Theta^{\prime}$ and model update $\widetilde{\nabla}^{t}$ 
    \begin{algorithmic}[1]
    \If{$t=0$}
        \State $s_{0}^{(1)}\leftarrow g^{0}, s_{0}^{(2)}\leftarrow g^{0}$
        \textcolor{blue}{\Comment{initialize values }}
    \Else
        \State Split $\mathbb{D}$ into train set $\mathbb{D}^{\textnormal{train}}$ and validation set $\mathbb{D}^{\textnormal{val}}$
        \State Set $\widehat{g}^{t+1}\leftarrow \mathbb{P}(s_{t}^{(1)},s_{t}^{(2)})$
        \textcolor{blue}{\Comment{construct estimator}}
        \State Initialize malicious model $\Theta^{\prime} \leftarrow \widehat{g}^{t+1}$, certified radius $R\leftarrow 0$, average update $\overline{\nabla}^{t}\leftarrow\widehat{g}^{t+1}-g^{t}$
        \For{each malicious client $i=1$ \textbf{to} $m$}
            \State Set $\nabla_{i}^t\leftarrow \operatorname{NormalLocalTraining}(g^t,\mathcal{D}_{i})$
            \State $R=\operatorname{Max}\left(\left\|\nabla_{i}^t- \overline{\nabla}^{t}\right\|_2, R\right)$\textcolor{blue}{\Comment{\small{find certified radius}}}
        \EndFor
        \For{local epoch $e=1$ \textbf{to} $E$}
            \For{batch $b$ \textbf{in} $\mathbb{D}^{\textnormal{train}}$}
                \State Calculate $\mathcal{L}\left(\Theta^{\prime}, \mathbb{D}_b^{\textnormal{train}}\right)$ based on Eq.~(\ref{Eq:My_loss})
                \State $\Theta^{\prime} \leftarrow \Theta^{\prime} - \mathcal{A}\left(\nabla \mathcal{L}(\Theta^{\prime}, \mathbb{D}_b^{\textnormal{train}})\right)$
                \If {$\operatorname{Accuracy}(\Theta^{\prime},\mathbb{D}^{\textnormal{val}})\le \tau$}       
                    % \State Set $\widetilde{\nabla}_{t+1}\leftarrow \Theta'-g_{t}$
                    % \State \textbf{Output}  $\Theta'$ and $\widetilde{\nabla}_{t+1}$ 
                    \State \textbf{Goto} \textit{Line} $20$
                    \textcolor{blue}{\Comment{early stopping}}
                \EndIf
            \EndFor
        \EndFor
%\While {$acc>\tau$}
        \State Set $\widetilde{\nabla}^{t} \leftarrow \Theta^{\prime} -g^{t}$ \textcolor{blue}{\Comment{calculate malicious update}}

        % \State $\widetilde{\nabla}^{t} \leftarrow \frac{\widetilde{\nabla}^{t}}{\operatorname{Max} \left(1,\left\|\widetilde{\nabla}^{t}-\overline{\nabla}^{t}\right\|_2 / R\right)}$

        \State 
        $\widetilde{\nabla}^{t}=\prod_{\overline{\nabla}^{t}+\left\{\delta \in \mathbb{R}^d:\|\delta\|_2 <R\right\}}\left(\widetilde{\nabla}^{t}\right)$
        \textcolor{blue}{\Comment{project it on ball}}
        \State \textbf{Output} $\Theta^{\prime}$ and $\widetilde{\nabla}^{t}$ 
    \EndIf
    \State Update $s_{t+1}^{(1)}$ and $s_{t+1}^{(2)}$ according to  Eq.~(\ref{Eq:Predict})
\end{algorithmic}
\end{algorithm}

%The generalized certified radius has been used as the measure of defense strength. And a smaller certified radius improves robustness (against poisoning attack), since models which are close to each other may well infer the same label for the same data point.  
% Overall, directly performing similarity detection on locally committed updates to discriminate suspicious updates from benign ones is the most widely accepted method in the related literatures.

As mentioned in Section~\ref{Sec:Defense}, we focus on two main categories of defenses, i.e., \textit{statistics-based} and \textit{distance-based}. Since models close to each other may well infer the same label for the same data point, more similar models imply better robustness (against poisoning attack)~\cite{pmlr-v151-panda22a}.
In these categories of defenses, outliers of uploaded model updates (which we describe with the \textbf{certified radius} $\mathcal{R}$ in the proof) are likely to be considered as suspected attackers and those updates will be reprocessed (e.g., eliminated, purified, corrected) before aggregation. However, our attack ensures that poisoned models are close to the clique of benign ones. As practiced in~\cite{ViratShejwalkar2022ManipulatingTB}, such a fundamental attack principle achieves good performance because it can seriously confuse defenses by avoiding making itself an outlier. Further, \ct{I-FMPA} searches the best-influential poisoned model in the neighborhood of the benign reference model, while the search distance is upper bounded by the maximum distance between any normally trained model (available to the attacker) and the reference model. We also provide the theoretical analysis of \ct{I-FMPA}.

\begin{definition}[$\varrho$-Corrupted Poisoning Attack]
Let $\zeta:=\left\{\nabla_1^{t}, \ldots,\nabla_{k}^{t} \ldots,\nabla_{N}^{t}\right\}$ be the union of local updates $\nabla\in \mathbb{R}^d$ among all clients at round $t$. 
The set of poisoned unions $\mathcal{S}(\varrho)$ consists of $\zeta^\star:=\left\{\widetilde{\nabla}_1^{t}, \ldots,\widetilde{\nabla}_{m}^{ t},\nabla_{m+1}^{t},\ldots,\nabla_{k}^{t},\ldots,\nabla_{N}^{t}\right\}$, which is identical to $\zeta$ except the updates $\widetilde{\nabla}_{\{i\in [m]\}}^{t}$ is a $\varrho$-corrupted version of $\nabla_{\{i\in [m]\}}^{t}$. That is, for any round $t$, we have $\zeta^\star=\zeta+\epsilon$ for certain $\epsilon$ with $\left\|\epsilon\right\|_2\leq\varrho$.
%Let the set of poisoned unions $\mathcal{S}(\varrho)$ be all union $\zeta^\star:=\left\{\widetilde{\nabla}_1^{t}, \ldots,\widetilde{\nabla}_{m}^{ t},\nabla_{m+1}^{t},\ldots,\nabla_{k}^{t},\ldots,\nabla_{N}^{t}\right\}$ identical to $\zeta$ except the updates $\widetilde{\nabla}_{\{i\in [m]\}}^{t}$ is a $\varrho$-corrupted version of $\nabla_{\{i\in [m]\}}^{t}$. That is, for any round $t$, we have $\zeta^\star=\zeta+\epsilon$ for certain $\epsilon$ with $\left\|\epsilon\right\|_2\leq\varrho$.
\end{definition}

\begin{definition}[Certified Radius] Suppose $\zeta$ is a union,
 the mean of the updates among all clients is $\overline{\nabla}^t=\frac{1}{N} \sum_{i \in[1,N]} \nabla_i^t$. We call $\mathcal{R}$ a certified radius for $\zeta$ if 
$\bm{\sup}\left\{\left\|\nabla_i^t-\overline{\nabla}^t\right\|_2: i \in [N]\right\}=\mathcal{R}$. For simplicity, we denote it as $\zeta \models\mathcal{R}$.
\end{definition}

\begin{lemma}
For a union and it associated certified radius $\zeta\models\mathcal{R}$, we assume that $\zeta^\star\in \mathcal{S}(\varrho)$ is a $\varrho$-corrupted version of $\zeta$ obtained by Algorithm~\ref{alg:algorithm1}.
Algorithm~\ref{alg:algorithm1} guarantees that, for any $\zeta^\star$, $\zeta^\star$ $\models$ $\mathcal{R}$ is still satisfied.
\end{lemma}

\begin{proof}
Algorithm~\ref{alg:algorithm1} guarantees that the distance between poisoned updates $\widetilde{\nabla}_{\{{i\in[m]}\}}^{t}$  and benign average update $\overline{\nabla}^{t}$ are upper bounded by the certified radius $\mathcal{R}$ by projecting the poisoned update onto a ball of radius $\delta$ ($\|\delta\|_2<\mathcal{R}$) around $\overline{\nabla}^{t}$, i.e., $\left\|\mathbb{E}[\widetilde{\nabla}_{\{i\in[m]\}}^{t}]-\overline{\nabla}^{t}\right\|_2 < R$.
\end{proof}

\begin{theorem}
\label{Theorem:1}
From lemma 1, in union $\zeta^\star$, there must be a benign client's update $\nabla_{k}^{t}$ ($k\in[m+1,N]$) such that $\left\|\mathbb{E}[\widetilde{\nabla}_{\{i\in[m]\}}^{t}]-\overline{\nabla}^{t}\right\|_2<\left\|\nabla_k^t-\overline{\nabla}^{t}\right\|_2$. If $\nabla_{k}^{t}$ is removed from the  $\zeta^\star$, we have $\zeta^\star$ $\models$ $(R-\epsilon)$, where $\epsilon$ is a small positive constant value. That is,
$\left\|\mathbb{E}[\widetilde{\nabla}_{\{i\in[m]\}}^{t}]-\overline{\nabla}^{t}\right\|_2\leq \mathcal{R}-\epsilon <\left\|\nabla_k^t-\overline{\nabla}^{t}\right\|_2 \leq R$.
\end{theorem}

%A typical robust aggregation algorithm $F_{\agr}(.)$ is performed in at least two steps, i.e., outlier removal or correction followed by aggregation.
% \begin{remark}
% We Consider a to-be-aggregated union $\zeta_B^\star=\zeta_A^\star \cup \{\textnormal{Client}_{k}\}$, where $\zeta_A$ $\models$ $(R-\epsilon,\varrho)$-CR. As per~\cite{ViratShejwalkar2022ManipulatingTB}, any update outside the certified radius $(\mathcal{R-\epsilon})$ like
% $\textnormal{Client}_{k}$ is likely to be considered most Byzantine-suspicious by $F_{\agr}(.)$ and removed from the to-be-aggregated candidate set.
% \end{remark}

\begin{example}
Consider a scenario where the defense knows the number of attackers (e.g., 20\%). Then, for safe aggregation, at least 20\% of the elements in union $\zeta^\star$ will be reprocessed. Theorem~\ref{Theorem:1} shows that there exist benign updates with the largest deviation (i.e., outliers) from the mean. They have a higher probability of being reprocessed by defenses compared to our poisoned updates.
\end{example}

\begin{table*}[ht]  
   \centering
   \begin{threeparttable}[b]
    \resizebox{0.8\textwidth}{!}{
    \begin{tabular}{c|c|c||c|c|c|c|c|c|c}
    \hline
    \multicolumn{1}{c|}{\multirow{1}[4]{*}{\makecell[c]{\textbf{Dataset} \\ \textbf{(Model)}}}} & \multirow{1}[4]{*}{\textbf{Defense}} & \multirow{1}[4]{*}{\makecell[c]{\textbf{No attack (\%)}}} & \textbf{\footnotesize{AGR-tailored}} & \multicolumn{6}{c}{\textbf{AGR-agnostic}} \\
\cline{4-10}          &       &       & \ct{AGRT} & \ct{LIE} &\ct{IPM}  & \ct{MPAF} & \ct{\footnotesize{Min-Max}} & \ct{\footnotesize{Min-Sum}} & \ct{\footnotesize{I-FMPA}}\\
    \hline
    \hline
    \multicolumn{1}{c|}{\multirow{8}[2]{*}{\makecell[c]{CIFAR-10 \\ (Resnet20)}}} & \ct{Krum}  & 67.41  & 25.43  & 37.55  & 29.26  &  32.83 & 25.36 & 48.30 & \cellcolor[HTML]{EFEFEF}\textbf{56.10 }\\
          & \ct{Mkrum} & 84.53  & 19.57  & 35.11  & 23.94  &  27.51 & 45.54 & 48.92 & \cellcolor[HTML]{EFEFEF}\textbf{53.36}\\
          & \ct{Bulyan} & 83.78  & 26.80  & 47.24  & 16.85  &  21.02 & 48.93 & 55.56 & \cellcolor[HTML]{EFEFEF}\textbf{58.26}\\
          & \ct{Median} & 80.62  & 25.37  & 38.05  &  31.22 &  37.63 & 51.22 & \textbf{56.74} & \cellcolor[HTML]{EFEFEF}51.43\\
          & \ct{TrMean} & 84.20  & 28.51  & 35.70  & 24.37  &  38.26 & 53.27 & 21.97 & \cellcolor[HTML]{EFEFEF}\textbf{57.82}\\
           & \ct{CC}  & 83.88  & \textit{—} & 13.77 & 12.38 & 17.63 & 29.15  & 22.49  & \cellcolor[HTML]{EFEFEF}\textbf{33.51}\\
          & \ct{AFA}   & 82.57  & \textit{—} & 16.38  &  7.55 &18.07   & 41.34 & 37.34 & \cellcolor[HTML]{EFEFEF}\textbf{48.70}\\
          & \ct{DNC}  & 84.15  & \textit{—} & 3.63  & 2.59 & 3.08  & 5.18 & 4.06 & \cellcolor[HTML]{EFEFEF}\textbf{9.77}\\
    \hline
    \hline
    \multicolumn{1}{c|}{\multirow{8}[2]{*}{\makecell[c]{MNIST \\ (FC)}}}  &\ct{Krum}  & 88.92  & 20.77  & 10.61  &  17.91 & 5.84 & 3.79 & \textbf{25.56} & \cellcolor[HTML]{EFEFEF}19.36\\
          & \ct{Mkrum} & 96.26  & 11.03  & 3.43  & 6.14  & 12.05  & 15.61 & 13.11 & \cellcolor[HTML]{EFEFEF}\textbf{21.08 }\\
          & \ct{Bulyan} & 95.38  & 7.44  & 7.02  &  4.30 & 7.81  & 5.35 & 8.07 & \cellcolor[HTML]{EFEFEF}\textbf{10.62}\\
          & \ct{Median} & 93.82  & 1.61  & 1.93  & 3.14  & \textbf{12.85} & 7.65  & 3.02 & \cellcolor[HTML]{EFEFEF}5.97\\
          & \ct{TrMean} & 96.74  & 1.77  & 2.30  &  1.88 & 9.67  & 9.25 & 8.84 & \cellcolor[HTML]{EFEFEF}\textbf{13.62 }\\
          & \ct{CC}  & 95.71  & \textit{—} & 1.48  &0.96 & 1.50 & 1.80  & 2.11  & \cellcolor[HTML]{EFEFEF}\textbf{5.76}\\
          & \ct{AFA}   & 96.35  & \textit{—} & 2.29  & 1.13  & 2.84  & 3.23 & 4.06 & \cellcolor[HTML]{EFEFEF}\textbf{8.64}\\
          & \ct{DNC}  & 95.97  & \textit{—} & 0.21  & 0.39 & 1.21  & 0.35 & 1.26 & \cellcolor[HTML]{EFEFEF}\textbf{2.15}\\
    \hline
    \hline
    \multicolumn{1}{c|}{\multirow{8}[2]{*}{\makecell[c]{EMNIST \\ (CNN)}}} & \ct{Krum}  & 61.93  & 16.65  & 2.53  &  15.38 &  19.32 & 4.81 & 18.83 & \cellcolor[HTML]{EFEFEF}\textbf{36.93}\\
          & \ct{Mkrum} & 80.21  & 22.85  & 14.61  & 12.60 & 26.77 & 52.72 & 37.93 & \cellcolor[HTML]{EFEFEF}\textbf{63.97}\\
          & \ct{Bulyan} & 80.24  & 18.27  & 20.78 & 10.42 &  19.81 & 17.23 & 24.55 & \cellcolor[HTML]{EFEFEF}\textbf{25.79}\\
          & \ct{Median} & 72.16  & 8.36  &  17.69 & 22.75 & 24.93 & 24.36 & 19.20& \cellcolor[HTML]{EFEFEF}\textbf{36.47}\\
          & \ct{TrMean} & 79.86  & 4.96  & 13.71  & 4.51  &  16.90 & \textbf{26.70} & 17.43 & \cellcolor[HTML]{EFEFEF}23.62\\
          & \ct{CC}  & 80.25  & \textit{—} & 12.44 &13.59 & 8.19 & 15.83  & 19.58   & \cellcolor[HTML]{EFEFEF}\textbf{25.46}\\
          & \ct{AFA}   & 78.92  & \textit{—} & 9.05  & 6.27  &24.06   & 35.75 & 32.28 & \cellcolor[HTML]{EFEFEF}\textbf{48.40}\\
          & \ct{DNC}  & 80.31  & \textit{—} & 6.80  & 5.47 & 8.05 & 15.69 & 11.62 & \cellcolor[HTML]{EFEFEF}\textbf{19.54}\\
    \hline
    \end{tabular}}
   %  \begin{tablenotes}
   %   \item[1] Note that the missing values indicate that \ct{AGRT}~\cite{MinghongFang2019LocalMP} did not tailor attack schemes for the corresponding defenses in their original work. Therefore, we do not compare them.
   % \end{tablenotes}
    \caption{Comparison of the \textit{attack impact} $\phi$  between SOTA MPAs and \ct{I-FMPA}.}
    \label{tab:AttackImpact}
    \end{threeparttable}
\end{table*}

\subsubsection{Fine-Grained Control Attack (\ct{F-FMPA})}

As emphasized in the Introduction, a new feature of our \ct{\ourmethod} is to enable the adversary to precisely control the performance drop of the central aggregated model. For instance, when the FL training is finished, the aggregated model $g^\prime$ (in the presence of fine-grained model poisoning) is inferior to model $g$ (without model poisoning) with $10\%$ drop in accuracy. 
Considering the fact that the FL central server knows neither the final accuracy of the converged model nor the clients' data, it is difficult, if not impossible, for the server to even realize that it is suffering from MPA.
In particular, competing FL service providers can easily gain an advantage with such precise and subtle MPAs, hence opening a new attack surface for FL.

%Fine-Grained Control
To perform our fine-grained control attack (precisely manipulating the MPA effects), the adversary will set $\tau$ to his desired accuracy level, say $\tau = 80\%$, and call Algorithm~\ref{alg:algorithm1} to produce a malicious model $\Theta^{\prime}$ (skip execution of lines 7-10 and 21). The remaining key issue is how to design model updates $\widetilde{\nabla}_{\textnormal{Precise}}$ for malicious clients such that the aggregation result will be exactly $\Theta^{\prime}$. This can be achieved analytically by looking at the server-side aggregation. Specifically, for FedAvg with and without MPA, we conclude:
\begin{IEEEeqnarray}{rCL}
   g^{t+1}  &=& g^t - \eta (m \cdot \nabla + \sum\nolimits_{i=m+1}^N  \nabla^i)/N,  \nonumber \\ 
   \Theta^{\prime} &=& g^t - \eta (m \cdot \widetilde{\nabla}_{\textnormal{Precise}} + \sum\nolimits_{i=m+1}^N  \nabla^i)/N, 
\end{IEEEeqnarray}
where $\nabla$ is the averaged model update of all the malicious clients without launching MPA. 
In view of this fact, it is easy to see the malicious model update is:
\begin{IEEEeqnarray}{rCL}
\widetilde{\nabla}_{\textnormal{Precise}} = \nabla + \frac{N}{\eta m}(g^{t+1} -\Theta^{\prime}).
\end{IEEEeqnarray}

Replacing the unknown $g^{t+1}$ with the predicted model $\widehat{g}^{t+1}$ using Eq.~(\ref{Eq:Predict}), the goal of precisely controlling the MPA by attackers is achieved. In the case that $\rho=\frac{N}{\eta m}$ is unknown, \ct{F-FMPA} search for a suitable scaling factor
$\rho$ by iteratively increasing it each round and validate the effect on $\mathbb{D}^\text{val}$. Experiments show that \ct{F-FMPA} also generalizes well to other robust AGRs (especially those based on filtered averaging).

\begin{figure*}[ht]
    \centering    \includegraphics[width=\textwidth]{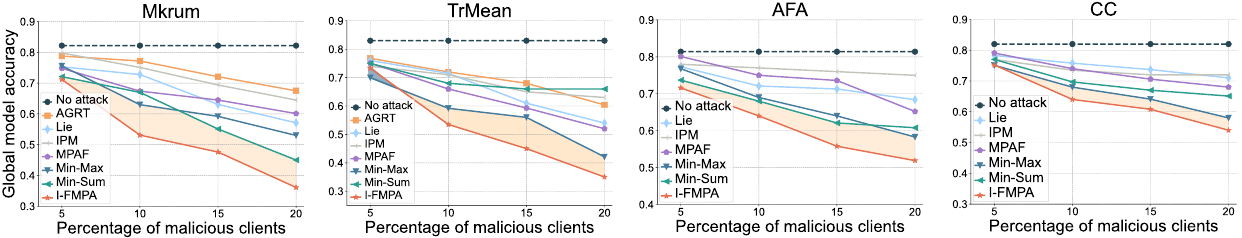}
    \caption{The global model accuracy against increasing percentage of malicious clients under different AGRs.}
    \label{Fig:PercentageOfAttackers}
\end{figure*}

\section{Experimental Results and Analyses}
\subsection{Experimental Setup}

\paragraph{Datasets and Models.}
We conduct experiments on the following datasets and models,  MNIST~\cite{YannLeCun1998GradientbasedLA} with a fully connected network (FC), EMNIST~\cite{GregoryCohen2017EMNISTEM} with a CNN, and CIFAR-10 with Resnet20/VGG-16, to evaluate \ct{\ourmethod} under various settings.

\paragraph{Parameter and Attack Settings for FL.}
Unless otherwise specified, we assume $20\%$ malicious clients in all adversarial setups. Given that poisoning FL with IID data is the hardest~\cite{MinghongFang2019LocalMP}, following existing work~\cite{ViratShejwalkar2022ManipulatingTB}, our evaluation mainly considers the case where the data is IID distributed among clients (EMNIST dataset is set to be non-IID). We also investigate the effects of different non-IID degrees on \ct{\ourmethod}.

\paragraph{Poisoning Attacks for Comparison.} 
% As introduced in Section~\ref{Sec:AttackIntroduce},
We adopt the following MPAs for comparison, i.e., \ct{AGRT}~\cite{MinghongFang2019LocalMP}, \ct{LIE} ~\cite{GiladBaruch2019ALI}, \ct{IPM}~\cite{pmlr-v115-xie20a}, \ct{MPAF}~\cite{XiaoyuCao2023MPAFMP}, and \ct{Min-Max} \& \ct{Min-Sum} ~\cite{ViratShejwalkar2022ManipulatingTB}.

\paragraph{Evaluated Defenses.}
We consider eight classical defenses, i.e., \ct{Krum} \& \ct{Mkrum}~\cite{PevaBlanchard2017MachineLW}, \ct{Median}~\cite{DongYin2018ByzantineRobustDL}, \ct{Trmean} (Trimmed-mean)~\cite{DongYin2018ByzantineRobustDL}, \ct{Norm-bounding}~\cite{AnandaTheerthaSuresh2019CanYR}, \ct{Bulyan}~\cite{ElMahdiElMhamdi2018TheHV}, \ct{FABA}~\cite{QiXia2019FABAAA}, and \ct{AFA}~\cite{LuisMuozGonzlez2019ByzantineRobustFM},  as well as two newly proposed defenses, i.e., \ct{CC} (Centered Clip)~\cite{pmlr-v139-karimireddy21a} and \ct{DNC}~\cite{ViratShejwalkar2022ManipulatingTB}.

\paragraph{Measurement Metrics.}
Let $\textnormal{acc}_\textnormal{benign}$ represent the accuracy of the optimal global model without attacks and $\textnormal{acc}_\textnormal{drop}$ represent the accuracy degradation caused by attacks.
We define \textit{attack impact} $\phi$ = $\textnormal{acc}_\textnormal{drop} / \textnormal{acc}_\textnormal{benign}$ $\times 100\%$ as the degree of performance degradation of the global model. Obviously, for a given attack, the larger $\phi$ is, the better the attack effect is.

\subsection{Evaluation Results}
We compare \ct{I-FPMA} with the previously described MPAs and the results are given in Tab.~\ref{tab:AttackImpact}. The column \textit{No attack} reports the global accuracy without attacks, while the remaining columns show the \textit{attack impact} $\phi$ under various AGRs. In particular, \ct{AGRT} is AGR-tailored (AGR is known to adversaries), while the others are all AGR-agnostic (AGR is unknown to adversaries). 
For fairness, all attacks except \ct{LIE} are performed without knowledge of updates from benign clients (\ct{LIE} uses statistics from all updates).
From Tab.~\ref{tab:AttackImpact}, the \textit{attack impact} $\phi$ of \ct{I-FMPA} is on average $\boldsymbol{2.4\times}$, $\boldsymbol{2.7\times}$, $\boldsymbol{3.5\times}$ and $\boldsymbol{2.8\times}$ higher than \ct{AGRT}, \ct{LIE}, \ct{IPM}, and \ct{MPAF}, respectively.
Under the same assumptions, \ct{I-FMPA} outperforms \ct{Min-Max/Sum} in most cases, showing that our attack imposes a stronger threat for causing DoS than SOTA MPAs.

Besides DoS attacks, we evaluate how precisely \ct{F-FMPA} can control the attack effect. For simplicity, we define the adversary's expected accuracy drop is $\xi$ (usually between $0-10\%$). Tab.~\ref{tab:PreciseAttack} shows the \textit{attack deviation} (i.e., absolute value of the difference between expected and actual accuracy drop) under different combinations of defenses and $\xi$.  We observe that \ct{F-FMPA} generalizes well to other defenses, showing that $\textbf{22}$ out of $\textbf{27}$ ($81.48\%$) combinations achieve an \textit{attack deviation} of less than 1\%. 
%That is, the adversary can manipulate final FL accuracy almost arbitrarily without being noticed. Thus this subtle attack is enough to make FL service providers fail in the commercial competition.

\begin{table}[bp]
\resizebox{0.45\textwidth}{!}
{
\begin{tabular}{ccccc}
\hline
\multirow{2}{*}{\begin{tabular}[c]{@{}c@{}}\textbf{Aggregation} \\ \textbf{algorithm}\end{tabular}} & \multirow{2}{*}{\begin{tabular}[c]{@{}c@{}}\textbf{No attack}\\ \textbf{(\%)}\end{tabular}} & \multicolumn{3}{c}{\textbf{Attack deviation (\%)} \bm{$\downarrow$}} \\ \cline{3-5} 
 &  & \bm{$\xi=2.5\%$} & \bm{$\xi=5\%$} & \bm{$\xi=10\%$} \\ \hline
\ct{FedAvg} & 84.70  & 0.17 {\scriptsize ($\pm$0.08)} & 0.23 {\scriptsize ($\pm$0.11)}& 0.43 {\scriptsize ($\pm$0.18)} \\ 
\ct{Norm-bounding} & 84.48 & 0.31 {\scriptsize ($\pm$0.15)}& 0.38 {\scriptsize ($\pm$0.19)}& 0.62 {\scriptsize ($\pm$0.30)} \\
\ct{Mkrum} & 84.53 & 0.35 {\scriptsize ($\pm$0.22)}& 0.65 {\scriptsize ($\pm$0.30)}& 0.71 {\scriptsize ($\pm$0.34)} \\ 
\ct{Median} & 80.62 & 0.50 {\scriptsize ($\pm$0.29)}& 1.39 {\scriptsize ($\pm$0.45)}& 2.17 {\scriptsize ($\pm$0.72)}\\ 
\ct{Bulyan} & 83.78 & 0.34 {\scriptsize ($\pm$0.22)}& 0.74 {\scriptsize ($\pm$0.34)} & 1.45 {\scriptsize ($\pm$0.53)}\\ 
\ct{TrMean} & 84.20 & 0.36 {\scriptsize ($\pm$0.14)}& 0.79 {\scriptsize ($\pm$0.26)} & 0.86 {\scriptsize ($\pm$0.34)}\\ 
\ct{CC} & 83.88 & 0.44 {\scriptsize ($\pm$0.19)}& 0.87 {\scriptsize ($\pm$0.31)}&  0.95 {\scriptsize ($\pm$0.38)}\\ 
\ct{AFA} & 82.57  & 0.39 {\scriptsize ($\pm$0.20)} & 0.65 {\scriptsize ($\pm$0.28)} & 0.88 {\scriptsize ($\pm$0.36)} \\ 
\ct{DNC} & 84.15 & 0.58 {\scriptsize ($\pm$0.40)}& 1.21 {\scriptsize ($\pm$0.51)} & 3.29 {\scriptsize ($\pm$0.86)} \\ \hline
\end{tabular}
}
\caption{\textit{Attack deviation} of \ct{F-FMPA} under different AGRs and $\xi$.}
\label{tab:PreciseAttack}
\end{table}
%When the AGR is FedAvg~\cite{HBrendanMcMahan2016CommunicationEfficientLO}, the attack can achieve an APR higher than $97\%$ in all cases, which shows that the attacker can manipulate final FL accuracy almost arbitrarily without being noticed. 
% Even when Byzantine-robust AGRs are adopted (Norm-bounding~\cite{AnandaTheerthaSuresh2019CanYR} and Mkrum), \ct{F-FMPA} can precisely manipulate the accuracy drop of the global model within $20\%$ loss of APR. 

\subsection{Ablation Studies}
\paragraph{Impact of the percentage of attackers.~}
{Fig.~\ref{Fig:PercentageOfAttackers}} shows the impact of various MPAs as the percentage of attackers changes from $5\%$ to $20\%$ on CIFAR-10 with Resnet20. We note that, due to the difference in design rationale, \ct{FMPA} is always superior to existing attacks for most combinations of malicious client percentages, defenses, datasets, and models.

\paragraph{Effects of the reference model.} We compared the attack effect of three reference models (i.e., HRM, ARM and PRM). The results are shown in Fig.~\ref{Fig:ReferenceModel}. Among them, the performance of HRM is always poor. As expected, ARM does not have an advantage when there are fewer malicious clients. The proposed PRM performs consistently well, revealing it to be a cost-effective option.

\paragraph{Impact of the non-IID degree.} 
We generate the non-IID MNIST using the method in~\cite{MinghongFang2019LocalMP} and evaluate the performance of MPAs under different non-IID degrees. The \textit{attack impact} of all MPAs increases as non-IID degree increases, and \ct{I-FMPA} performs better than SOTA MPAs.

\begin{figure}[t]
    \centering    \includegraphics[width=0.5\textwidth]{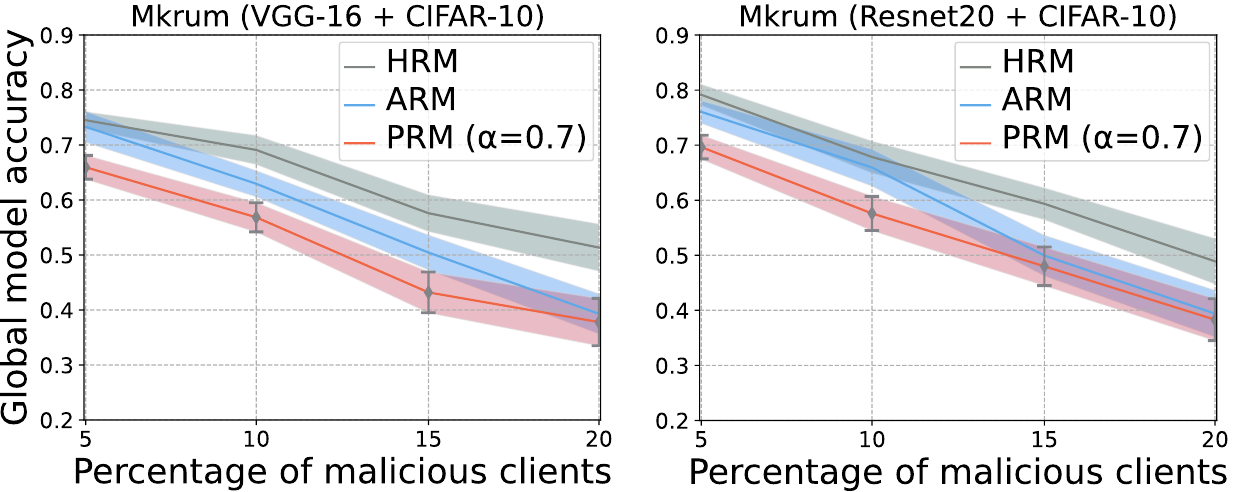}
    \caption{Impacts of three different reference models on attack effect when malicious clients are armed with \ct{I-FMPA} ($\alpha=0.7$).}
    \label{Fig:ReferenceModel}
\end{figure}

\begin{figure}[t]
\begin{center}
\includegraphics[width=0.5\textwidth]{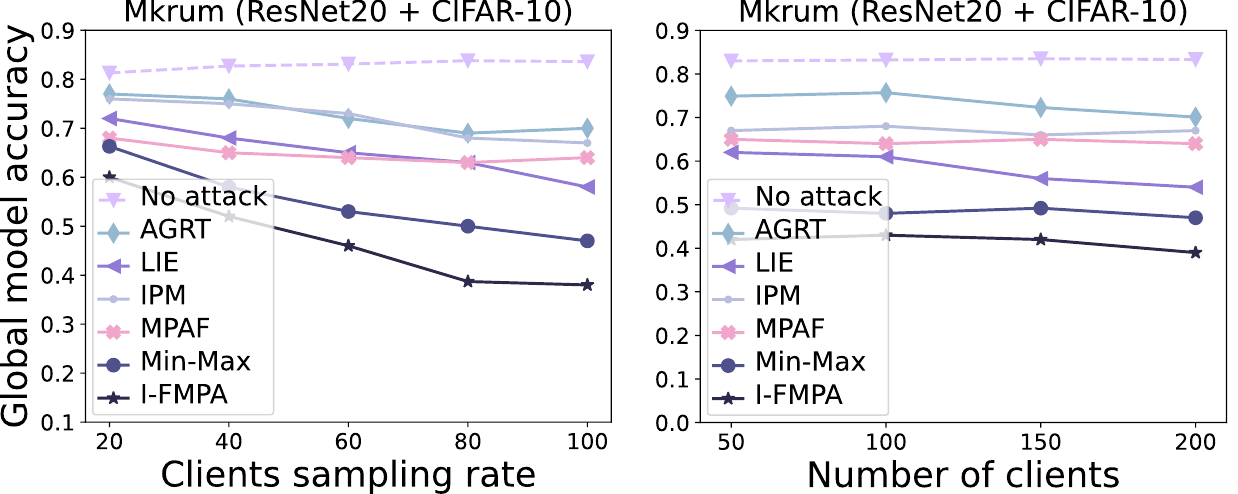}
\end{center}
\caption{Testing the effect of various MPAs as we vary the clients sampling rate and total number of clients.}
\label{Fig:ClientSampling}
\end{figure}

\paragraph{Different FL scenarios.}
Fig.~\ref{Fig:ClientSampling} shows that all MPAs become less effective when FL uses a random client sampling strategy, because the adversary cannot consistently corrupt the global model. But \ct{I-FMPA} still performs the best with different total client counts and sample rates.

%在固定频率攻击中，更容易看出攻击是否发生在特定的回合中
\section{Conclusion}
This paper introduces a novel model poisoning framework: \ct{\ourmethod}. To overcome the inefficiency of attacks caused by the prescribed perturbation direction in existing works, our \ct{I-FMPA} utilizes historical data in FL to iteratively optimize the malicious model through neuron perturbations to launch more effective DoS attacks. It outperforms the SOTA attacks, albeit with weaker assumptions. We also propose \ct{F-FMPA}, opening up a new attack surface in FL MPAs. It enables malicious service providers to surpass competitors stealthily.
Currently, \ct{\ourmethod} still utilizes data from malicious clients. We leave it as future work to explore \ct{\ourmethod} on zero-shot attacks.

\section*{Acknowledgments}
This work was supported in part by National Natural Science Foundation of China under Grant No. 62032020, National Key Research and Development Program of China under Grant 2021YFB3101201, the Open Fund of Science and Technology on Parallel and Distributed Processing Laboratory (PDL) under Grant WDZC20205250114. Leo Yu Zhang is the corresponding author.

\bibliographystyle{named}
\bibliography{ijcai23}

\end{document}